\documentclass{article}

 \usepackage[preprint]{neurips_2026}

\usepackage[utf8]{inputenc} % allow utf-8 input
\usepackage[T1]{fontenc}    % use 8-bit T1 fonts
\usepackage{hyperref}       % hyperlinks
\usepackage{url}            % simple URL typesetting
\usepackage{booktabs}       % professional-quality tables
\usepackage{amsfonts}       % blackboard math symbols
\usepackage{nicefrac}       % compact symbols for 1/2, etc.
\usepackage{microtype}      % microtypography
\usepackage{xcolor}         % colors

\usepackage{titlesec}
\usepackage{caption}
\usepackage{caption-light}
\usepackage{nicefrac}
\usepackage{bm}
\usepackage{amsmath}
\usepackage{amssymb}
\usepackage{tikz}
\usepackage{pgfplots}
\pgfplotsset{compat=1.18} % Use this to avoid compatibility warnings
\usetikzlibrary{positioning}
\usepackage{hyperref}
\hypersetup{
  hidelinks,
  pdfborder={0 0 0},
}

\usepackage{ifthen}

\usepackage{amsthm}

\usepackage{algorithm}
\usepackage{algpseudocode}

\newtheorem{theorem}{Theorem}

\newtheorem{remark}[theorem]{Remark}

\newtheorem{proposition}[theorem]{Proposition}
\newtheorem{assumption}[theorem]{Assumption}

\newtheorem{example}[theorem]{Example}
\usepackage{mathtools}
\usepackage{amssymb}
\usepackage{ifthen}
\usepackage{algorithm}

\newcommand{\Mat}[1][]{\ifthenelse{\equal{#1}{}}{\textnormal{Mat}}{\textnormal{Mat}(#1)}}

\newcommand{\tangent}[1]{
    \ifthenelse{\equal{#1}{}}
    {{T}}
    {{T_{#1}}}
}
\newcommand{\dualtangent}[1]{
    \ifthenelse{\equal{#1}{}}
    {{T^*}}
    {{T_{#1}^*}}
}

\newcommand{\Adjoint}[1]{
    \ifthenelse{\equal{#1}{}}
    {\textnormal{Ad}}
    {\textnormal{Ad}_{#1}}
}

\newcommand{\adjoint}[1]{
    \ifthenelse{\equal{#1}{}}
    {\textnormal{ad}}
    {\textnormal{ad}_{#1}}
}
\newcommand{\coAdjoint}[1]{
    \ifthenelse{\equal{#1}{}}
    {\textnormal{Ad}^*}
    {\textnormal{Ad}^*_{#1}}
}

\newcommand{\coadjoint}[1]{
    \ifthenelse{\equal{#1}{}}
    {\textnormal{ad}^*}
    {\textnormal{ad}^*_{#1}}
}

\newcommand{\ones}{\mathbf{1}}
\newcommand{\set}[1]{\left\{#1\right\}}

\newcommand{\abs}[1]{\left|#1\right|}

\newcommand{\transpose}{\intercal}

\newcommand{\Matrix}[1]{\mathrm{#1}}
\newcommand{\identity}[1]{
    \ifthenelse{\equal{#1}{}}
    {\Matrix{I}}
    {\Matrix{I}_{#1}}
}

\NewDocumentCommand{\Probability}{o m}{
  \mathsf{P}
  \IfValueT{#1}{_{#1}}
  \IfNoValueT{#1}{}
  \left[#2\right]
}
\NewDocumentCommand{\Expectation}{o m}{
  \mathsf{E}
  \IfValueT{#1}{_{#1}}
  \IfNoValueT{#1}{}
  \left[#2\right]
}
\NewDocumentCommand{\Variance}{o m}{
  \mathsf{Var}
  \IfValueT{#1}{_{#1}}
  \IfNoValueT{#1}{}
  \left[#2\right]
}
\NewDocumentCommand{\Covariance}{o m}{
  \mathsf{Cov}
  \IfValueT{#1}{_{#1}}
  \IfNoValueT{#1}{}
  \left[#2\right]
}

\newcommand{\Gaussian}{\mathcal{N}}

\newcommand{\onehot}[2][]{\mathsf{e}_{#2}\ifthenelse{\equal{#1}{}}{}{^{(#1)}}}

\newcommand{\initial}{0}
\newcommand{\horizon}{T}
\newcommand{\order}{\tau}

\newcommand{\stateProcess}{X}
\newcommand{\observationProcess}{Z}

\newcommand{\control}{u}

\newcommand{\observationModel}{C}
\newcommand{\processNoise}{B}
\newcommand{\observationNoise}{W}
\newcommand{\covarianceProcessNoise}{Q}
\newcommand{\covarianceObservationNoise}{R}

\newcommand{\dualstate}{y}

\newcommand{\covariance}{\Sigma}
\newcommand{\initialMean}{\mu_{\initial}}
\newcommand{\initialCovariance}{\covariance_{\initial}}
\newcommand{\dualCost}{\mathsf{J}}

\def\Re{\mathbb{R}}

\def\Sec#1{Sec.~\ref{#1}}
\def\Fig#1{Fig.~\ref{#1}}

\def\notes#1{\marginpar{\tiny #1}\typeout{Notes!
Notes!
Notes!
}}
\renewcommand{\notes}[1]{\typeout{notes!}}

\def\Re{\field{R}}

\def\Sec#1{Sec.~\ref{#1}}

\def\clP{{\cal P}}
\def\clZ{{\cal Z}}

\def\Sec#1{Sec~\ref{#1}}

\def\E{{\sf E}}

\def\Fig#1{Fig.~\ref{#1}}
\def\Sec#1{Sec.~\ref{#1}}

\def\clZ{{\cal Z}}

\def\beq{\begin{eqnarray}} 
\def\bc{\begin{center}} 
\def\be{\begin{enumerate}}
\def\bi{\begin{itemize}} 
\def\bs{\begin{small}}
\def\bS{\begin{slide}}
\def\ec{\end{center}} 
\def\ee{\end{enumerate}}
\def\ei{\end{itemize}}
\def\es{\end{small}}
\def\eS{\end{slide}}
\def\eeq{\end{eqnarray}}

\def\Re{\mathbb{R}}
\def\E{{\sf E}}

\def\Sec#1{Sec.~\ref{#1}}
\def\Thm#1{Thm.~\ref{#1}}
\def\Prop#1{Prop.~\ref{#1}}

\def\clP{{\cal P}}
\def\clZ{{\cal Z}}

\renewcommand{\Re}{\mathbb{R}}

\newcommand{\sfJ}{{\sf J}}
\def\clN{{\cal N}}

\def\clP{{\cal P}}

\def\clU{{\cal U}}

\def\clZ{{\cal Z}}

\def\E{{\sf E}}
\def\bS{\mathbb{S}}
\def\sJ{{\sf J}}

\def\ones{{\sf 1}}
\def\sP{{\sf P}}

\def\bO{\mathbb{O}}

\def\tp{\intercal}

\def\dvar{\operatorname{var}}

\def\opt{{\text{\rm (opt)}}}

\usepackage[dvipsnames]{xcolor}

\definecolor{illiniorange}{RGB}{255,95,5}
\definecolor{illiniblue}{RGB}{19,41,75}

\usepackage{algorithm}
\usepackage{algpseudocode}

\newtheorem{remark}[theorem]{Remark}

\title{Transformer-like Inference from Optimal Control}

\author{%
  Aditya Kudre\\
  Coordinated Science Laboratory\\
  Electrical and Computer Engineering\\
  University of Illinois Urbana-Champaign\\
  \And
  Heng-Sheng Chang\thanks{Corresponding author email: \texttt{hschang@illinois.edu}}\\
  Coordinated Science Laboratory\\
  Mechanical Science and Engineering\\
  University of Illinois Urbana-Champaign\\
  \And
  Prashant G. Mehta\\
  Coordinated Science Laboratory\\
  Mechanical Science and Engineering\\
  University of Illinois Urbana-Champaign\\
}

\begin{document}

\maketitle

\begin{abstract}
  Decoder-only transformers compute the conditional probability of the next token from a sequence of past observations. 
  This paper derives, from first principles, inference architectures that solve the same prediction problem --- and in doing so, recovers transformer-like layer operations as a consequence of optimal control theory. 
  The framework is developed for two model classes: a nonlinear model of discrete-valued processes, directly motivated by the transformer, and a linear Gaussian model as a tractable baseline. 
  For both model classes, the prediction objective is reformulated as an optimal control problem whose solution yields an explicit inference algorithm, the dual filter, with a layer structure that mirrors the layer structure of a decoder-only transformer.  
  Numerical experiments provide a comparison of the optimal control to attention weights from a trained transformer.  
  These experiments reveal that when the embedding dimension is insufficient, the transformer implicitly exploits non-Markovian structure.
\end{abstract}

\section{Introduction}\label{sec:intro}

A decoder-only transformer computes the conditional probability of the next token through a sequence of $L$ identical layer operations:
\begin{subequations}\label{eq:layer_definition}
  \begin{equation}\label{eq:layer_transformer}
  \text{(Layer for transformer)}\qquad [\sigma_1,\sigma_2,\hdots,\sigma_T]_{d\times T} \mapsto [\sigma_1^+,\sigma_2^+,\hdots,\sigma_T^+]_{d\times T},
  \end{equation}
where each layer maps a $(d\times T)$-sequence of (embedding) 
vectors to a $(d\times T)$-sequence of the same dimension, with $\sigma_t^+$
depending only on $\{\sigma_1,\hdots,\sigma_t\}$ (causal structure),
for $1\leq t\leq T$~\cite{phuong2022formal}.  
The output of the final layer is decoded to give the next-token prediction.
This layer structure is the defining architectural feature of the transformer --- yet its mathematical justification remains unclear. 
Why should the optimal prediction architecture take this particular form? 

This paper offers one answer using optimal control theory. 
% While recent work has utilized optimal control to establish a framework for transformer training \cite{akman2026optimal} and to improve the transformer robustness and efficiency \cite{kan2025optimal}, we apply these principles here to derive the inference architecture and layer operations from first principles.
We adopt a model-based approach where the observed tokens are generated from an underlying hidden stochastic process, and the prediction problem is cast as causal inference in a partially observed system:
\begin{align*}
  \text{(hidden state process)} &\quad X := [X_0,X_1,\hdots,X_{T}]
  \quad \text{taking values in state-space } \bS,\\
  \text{(observation process)} &\quad Z := [Z_1,Z_2\hdots,Z_{T},Z_{T+1}]
  \quad \text{taking values in observation-space } \bO,
\end{align*}
where the joint distribution of $(X,Z)$ is according to the probabilistic graphical model in \Fig{fig:model}. 
We consider two model classes, based on definition of $\bS$ and $\bO$, as described in Table~\ref{tab:lin_non_lin_problem}.

% The interpretation of transformers within a discrete-time dynamical
% system framework is a common perspective in the literature
% \cite{akman2026optimal,cimetiere2025localmax, krishnamurthy2026llms},
% and
The nonlinear model setting is directly motivated by the transformer: $\bO$ is the vocabulary, and $|\bS|=d$ is the embedding dimension. 
The prediction of interest is the conditional probability $\pi_T(\cdot) := \sP(X_T=\cdot \mid Z_{1:T})$, a probability vector of dimension $d$, from which the next-token prediction $\sP(Z_{T+1}=\cdot \mid Z_{1:T})$ is computed. 
The linear Gaussian model serves as a tractable baseline in which exact, closed-form results are available. 
The same notation $(X,Z)$ is used for both model classes; the appropriate class will be clear from context. 

Our objective is to derive transformer-inspired inference architectures where the prediction --- conditional probability for the nonlinear model and conditional expectation for the linear model --- is computed as a causal transformation of $Z_{1:T}$.  
The representations are formally introduced in Table~\ref{tab:lin_non_lin_solution}.

\begin{table}[t]
  \centering
  \renewcommand{\arraystretch}{1.25}
  \setlength{\tabcolsep}{5pt} % default is ~6pt
  \caption{%
    Two model classes: For the nonlinear model, the prediction is the conditional probability of the hidden state.
    For the linear Gaussian model, the prediction is the conditional expectation of the hidden state.
  }
  \begin{tabular}{c c c c c}
    \toprule
    \textbf{Models} & \textbf{State-space} $\bS$ & \textbf{Observation-space} $\bO$ 
      & \textbf{Prediction} & \textbf{Motivation} \\
    \midrule
    Nonlinear & $\{1,2,\hdots,d\}$ & $\{0,1,\hdots,m\}$ & $\pi_T := \sP(X_T|Z_{1:T})$ &  transformer-like \\ 
    % model & & & $x\in\bS$ & transformer \\[2ex]
    Linear & $\Re^d$ & $\Re^m$ & $\hat{X}_T := \E (X_T|Z_{1:T})$ & tractable baseline \\
    % model & & & & (Gaussian process) \\
    \bottomrule
  \end{tabular}\label{tab:lin_non_lin_problem}
\end{table}

\begin{table}[t]
  \centering
  \renewcommand{\arraystretch}{1.5}
  \setlength{\tabcolsep}{5pt} % default is ~6pt
  \caption{%
    Inference architectures for the two model classes: 
    The prediction is computed as a weighted sum of past observations. 
    In the nonlinear predictor, weights are causally adapted to the observation data (i.e., $U_t$ is allowed to depend upon $Z_{1:t}$). 
    This data-dependence is what makes the predictor nonlinear and is the direct analogue of attention weights in a transformer.%
  } 
  \begin{tabular}{c c c}
    \toprule
    & \textbf{Representation} & \textbf{Weights (or control)} \\
    \midrule
    Nonlinear predictor & $\pi_T(f) = \text{(const.)} - \sum_{t=1}^{T} U_{t-1}^\tp e(Z_t)$, & $U=\{U_0,\hdots,U_{T-1}\}$ is an \\
    (for nonlinear model) & for $f:\bS \to \Re$. & $\Re^m$-valued adapted process \\[1ex]
    Linear predictor & $f^\tp \hat{X}_T = \text{(const.)} - \sum_{t=1}^{T} u_{t-1}^\tp Z_t$, & $u=\{u_0,u_1,\hdots,u_{T-1}\}$ is an\\
   (for linear model)  & for $f\in\Re^d$. & $\Re^m$-valued deterministic process\\
    \bottomrule
  \end{tabular}\label{tab:lin_non_lin_solution}
\end{table}

A key finding is that for both model classes, the optimal predictor is computed by iterating a $d\times T$ sequence-to-sequence transformation --- structurally analogous to a transformer layer, but derived here from 
optimality conditions of the prediction problem rather than specified by architecture.

While the representation for the linear predictor is standard for the Gaussian processes, the representation for the nonlinear predictor is an original contribution of this paper (Prop.~\ref{prop:existence_nonlin_predictor_rep} gives well-posedness --- existence and uniqueness --- of $U$).  
The other original contributions are as follows:

\vspace*{-0.1in}
\paragraph{Optimal control framework for computing weights.} 
For both model classes, the synthesis of optimal weights --- $u$ for the linear predictor and $U$ for the nonlinear predictor --- is formulated as an optimal control problem (OCP).

For the linear model, the OCP is solved for the general non-Markovian $X$, yielding a non-recursive algorithm that is provably more efficient than the recursive Kalman filter architecture ($O(T^2d^2)$ vs $O(T^3d^2)$).  
For the nonlinear model, the OCP and explicit formula for $U$ are obtained only for the Markovian $X$, i.e., when $(X,Z)$ is a hidden Markov model (HMM), with complexity $O(T^2d^2)$ --- same as a transformer layer --- independent of vocabulary size $m$.

\vspace*{-0.1in}
\paragraph{Transformer-like layer operation.}
The OCP solution defines a layer transformation on a $d\times T$-sequence --- directly analogous to the transformer. 
For the linear model, the $(d\times T)$-sequence is the momentum process, and the layer transformation is defined using Pontryagin's maximum principle:
\begin{equation}\label{eq:layer_dual_filter_LG}
  \text{(Layer for linear predictor)} \quad [\eta_0,\hdots,\eta_{T-1}]_{d\times T}  \xrightarrow{\;\text{(Hamilton's equation)}\;} [\eta_0^+,\hdots,\eta_{T-1}^+]_{d\times T}.
\end{equation}
At convergence, the layer transformation yields the optimal momentum sequence, from which the prediction weights $u$ are computed via an explicit formula.

For the nonlinear model (analysis is limited to HMM), the layer maps the probability sequence: 
\begin{equation}\label{eq:layer_dual_filter}
  \text{(Layer for nonlinear predictor)} \quad [\rho_1,\hdots,\rho_T]_{d\times T}  \xrightarrow{\;\text{(optimality equation)}\;} [\rho_1^+,\hdots,\rho_T^+]_{d\times T}.
\end{equation}

At convergence, the layer transformation yields the conditional probability of the hidden state at each time step, from which the weights $U$ are computed via an explicit formula. 
For both models, the layer transformation emerges from the optimality conditions rather than being specified by architecture.

\end{subequations}

Numerical experiments provide a direct side-by-side comparison of optimal 
control weights and attention weights from a trained nanoGPT
(Fig.~\ref{fig:control-attention}). 
When the embedding dimension is insufficient, the dual filter 
degrades while nanoGPT retains near-optimal performance --- 
revealing that the transformer implicitly exploits non-Markovian 
structure.

% \vspace*{-0.1in}
\paragraph{Related work.}
Prior work on mathematical modeling of transformers interpret the
architecture as a transport on probability
measures~\citep{geshkovski2023mathematical,geshkovski2024measure,abella2024asymptotic,adu2024approximate,castin2025unified}. Building on these perspectives, optimal control formalisms have been applied to analysis and learning~\citep{akman2026optimal,kan2025optimal,cimetiere2025localmax,krishnamurthy2026llms}. A
complementary perspective is based on Bayesian inference including HMM
and nonlinear filtering
\citep{xie2021explanation,bai2023transformers,rohekar2023causal,alaa2019attentive,tang2021probabilistic,goel2024can,du2023can,chang2025dual}.
The present paper
builds directly on \citep{chang2025dual} but differs from other works,
which focus on interpreting and refining attention mechanisms rather
than modeling the prediction problem from first principles based on
the representations introduced in Table~\ref{tab:lin_non_lin_solution}.

The remainder of this paper is organized as follows: 
The OCPs for the linear and the nonlinear models are described in \Sec{sec:linear} and \Sec{sec:nonlinear}, respectively. 
The numerical experiments are presented in \Sec{sec:numerics} and the paper closes
with conclusions in \Sec{sec:conc}.

% Sec.~\ref{sec:linear} develops the linear Gaussian model, which serves as a tractable setting to introduce the duality framework and optimal control characterization; Sec.~\ref{sec:nonlinear} presents the parallel development for the nonlinear discrete-state model; Sec.~\ref{sec:algorithm} derives the dual filter algorithm and its correspondence to transformer architectures; and Sec.~\ref{sec:numerics} reports numerical experiments.

% The remainder of the paper is organized as follows: Sec.~\ref{sec:model} introduces the model classes and predictor representations, Sec.~\ref{sec:duality} develops the duality-based control formulations, Sec.~\ref{sec:algorithm} presents the dual filter algorithm, and Sec.~\ref{sec:numerics} reports numerical studies.

\begin{figure}[t]
  \centering
  \begin{tikzpicture}[scale=0.8, every node/.style={scale=0.6}, font=\footnotesize]
    \input{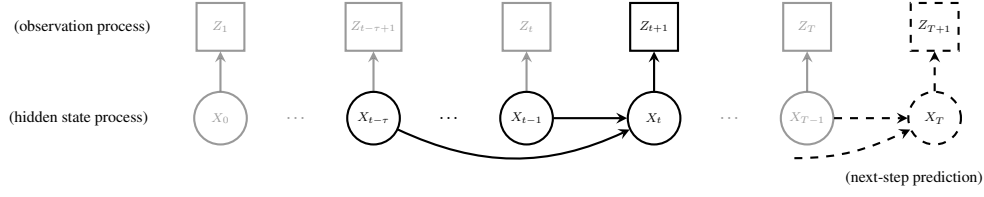}
  \end{tikzpicture}
  % \vspace*{0.75\baselineskip}
  \caption{%
    Graphical model for $(X,Z)$.  
    For $\tau=1$, $X$ is a Markov process and $(X,Z)$ is a hidden Markov model (HMM).  
    For $\tau>1$, $X$ is a non-Markovian (or a $\tau^{\text{th}}$-order Markov) process. For $\tau=T$, $X_T$ depends upon the entire past $X_{0:T-1}$.%
  }
  %\caption{
  %   Graphical representation of the non-Markovian model with order $\order$ and horizon $T$. 
  %   The state process $X_{0:T}$ is represented by circles and the observation process $Z_{0:T}$ is represented by squares. 
  %   The horizontal arrows (including curved ones) represent the causal and long-range dependencies between states, and the vertical arrows represent the emission of observations from states.
  %   The dashed arrows indicate the prediction task of estimating $Z_T$ given the past observations $Z_{0:T-1}$. 
  % }
  \label{fig:model}
\end{figure}

% Variational inference offers an alternative optimization-based approach to inference in latent variable models, with connections to control theory and potential implications for transformer design \cite{kingma2013auto, rezende2014stochastic, akman2026optimal}.

\section{Linear Gaussian Model}
\label{sec:linear}

Consider the graphical model in~\Fig{fig:model} for the case when $(X,Z)$ is an $\Re^d\times\Re^m$-valued Gaussian process, with state-space representation:
\begin{subequations}
  \begin{align}
    X_{t+1} &= \sum_{s=1}^{\min(\order,t+1)} A_{t+1,s} X_{t+1-s} + B_{t+1},\quad
            0 \leq t \leq T-1, \quad X_0 \sim \clN(\mu_0, \Sigma_0), \label{eq:linear-process}\\
    Z_{t+1} &= C X_{t} + W_{t+1} ,\quad 0\leq t \leq T, \label{eq:linear-observation}
  \end{align}%
  \label{eq:linear-model}%
\end{subequations}
where $\order$ is the model-order parameter, $C \in \Re^{m\times d}$ is the observation matrix, and $A_{\cdot,\cdot} \in \Re^{d\times d}$ are the state transition matrices. 
Stochasticity is introduced through three mutually independent sources: 
(1) the Gaussian initial condition $X_0 \sim \clN(\mu_0, \Sigma_0)$, 
(2) the white Gaussian noise (WGN) process $B_t \sim \clN(0, Q)$ with $Q\succeq 0$, and
(3) the WGN process $W_t \sim \clN(0, R)$ with $R\succ 0$, for $1\leq t\leq T$.

The goal is to compute the conditional expectation  $\hat{X}_T := \E(X_T \mid Z_{1:T})$, from which the next-step prediction $\E(Z_{T+1} | Z_{1:T}) = C \hat{X}_T$ is readily computed.

For the fully non-Markovian case (with $\tau=T$), a Kalman filter
needs to maintain an estimate, conditional mean and covariance, of the
entire history $\{X_0,X_1,\hdots,X_t\}$ at each time
$t$~\cite[Ch.~8.2]{bar2001estimation}.    
The size of this data structure grows with $t$ --- $O(td)$ for the
conditional mean and $O(t^2d^2)$ for the conditional covariance ---
leading to $O(T^3d^2)$ total complexity over the time-horizon (Table~\ref{tab:complexity}). 
This motivates the question: is a more efficient algorithm possible with a fixed $d\times T$ data structure, as in a transformer? 
The answer is affirmative --- the OCP formulation below yields such an algorithm.

\subsection{Optimal control problem (OCP) for computing weights}

The OCP is formulated as follows. 
Consider the space of admissible control inputs as 
\begin{equation*}
 \text{(admissible control)} \qquad \clU^{\text{(det)}} :=\{ u_t: u_t \in \Re^m,
  0\leq t\leq T-1 \}.
  \label{eq:admissible-control-linear}
\end{equation*}
For any given $f\in\Re^d$ and a deterministic control sequence $u\in\clU^{\text{(det)}}$, define the $\Re^d$-valued backward process $y=\{ y_t: y_t \in \Re^d, 0\leq t\leq T\}$ by
\begin{equation}
 \text{(dual control system)} \qquad y_t = \sum_{s=1}^{\min(\order,T-t)} A_{t+s,s}^\tp y_{t+s} + C^\tp u_t,
  \quad 0\leq t\leq T-1, \qquad y_T=f,
  \label{eq:linear-dual-backward}
\end{equation}
and the associated optimal-control cost functional
\begin{equation*}
  \label{eq:linear-cost}
 \text{(optimal control objective)} \qquad   \sJ_T(u;f) :=
 \frac{1}{2}\abs{y_0}_{\Sigma_0}^2 + \frac{1}{2} \sum_{t=0}^{T-1}
 \left( |y_{t+1}|_Q^2 + |u_t|_R^2\right).
\end{equation*}

\begin{proposition}[Duality principle for non-Markovian linear Gaussian model]
  Consider an estimator
  \begin{equation*}
    S_T := \mu_0^\tp y_0 - \sum_{t=1}^{T} u_{t-1}^\tp Z_{t},
    \label{eq:linear-estimator}
  \end{equation*}
  where $y_0$ is obtained from~\eqref{eq:linear-dual-backward} with $u \in\clU^{\text{(det)}}$ and $y_T=f$.
  Then
  \begin{equation}
    \sJ_T(u;f) = \E \left(\frac{1}{2}\abs{f^\tp X_T - S_T}^2\right).
    \label{eq:linear-duality}
  \end{equation}\label{prop:duality_LG}
\end{proposition}
\begin{proof}
See \Sec{appx:proof-duality-linear} in the appendix.
\end{proof}

The right-hand side of~\eqref{eq:linear-duality} is the mean-squared error (MSE). 
The duality principle relates the MSE to an optimal control objective: the control $u$ that minimizes $\sJ_T(u;f)$ yields the MMSE estimator, which for Gaussian processes equals the conditional mean $f^\tp\hat{X}_T$.
This motivates the following optimal control problem:
\begin{equation}
  \text{(OCP)} \qquad \min_{u \in\clU^{\text{(det)}}} \sJ_T(u;f) 
  \quad \text{subject to \eqref{eq:linear-dual-backward}}
  \label{eq:linear-optimal-control-problem}
\end{equation}

\begin{theorem}[Optimal control formula]\label{thm:optimal-control}
  Consider the OCP~\eqref{eq:linear-optimal-control-problem}.  
  Then there exists a unique optimal control sequence $u^\opt\in\clU^{\text{(det)}}$, given by the following forward-backward system:
  \begin{subequations}
   \label{eq:linear-fb-system}
    \begin{align}
      \text{(backward)}&& y_t &= \sum_{s=1}^{\min(\order,T-t)} A_{t+s,s}^\tp y_{t+s} + C^\tp u_t^\opt,\quad y_T=f,\label{eq:linear-fb-backward}\\
      \text{(forward)}&& \eta_t &= \sum_{s=1}^{\min(\order,t)} A_{t,s}\eta_{t-s} + Q y_t,\quad \eta_0=\Sigma_0 y_0,\label{eq:linear-fb-forward}\\
      \text{(optimal control formula)}&& u_t^\opt &= -R^{-1}C \eta_t,\quad t=0,1,\hdots,T-1.
      \label{eq:linear-fb-control}
    \end{align}  
  \end{subequations}\label{Thm:OCP_LG}
\end{theorem}
\begin{proof}
  See \Sec{appx:proof-optimal-control-linear} in the appendix. 
\end{proof}

\Thm{thm:optimal-control} justifies equation~\eqref{eq:layer_dual_filter_LG} for the layer operation introduced in \Sec{sec:intro}. 
The momentum process $\eta:=\{\eta_t:0\leq t\leq T-1\}$ is a $d\times T$ data structure, and a single backward-forward pass of~\eqref{eq:linear-fb-system}  defines the layer transformation $\eta\mapsto\eta^+$ in~\eqref{eq:layer_dual_filter_LG}.
Complexity comparison with the Kalman filter appears in Table~\ref{tab:complexity}. 
We refer to the algorithm as the dual filter.

\vspace*{-0.1in}
\paragraph{Related work.}
State space models such as Mamba invoke a duality between recurrent and convolutional representations to derive efficient sequential architectures \citep{gu2023mamba}. 
The duality in this paper is of a different character: it is between optimal filtering and optimal control~\cite{kim2019duality,kim2023duality}, yielding inference architectures from first principles rather than from architectural design.  
For the Markovian case ($\tau=1$), the duality is classical and referred to
as Kalman's or minimum-variance
duality~\cite[p.~180]{bensoussan2018estimation}~\cite[p.~100]{kailath2000linear},~\cite{todorov2008general}. See~\cite{kim2025arrow}
for a historical overview of duality in control theory.

This linear analysis sets the stage for the nonlinear model presented next.

\begin{table}[t]
  \centering
  \renewcommand{\arraystretch}{1.25}
  \setlength{\tabcolsep}{6pt}
  \caption{Complexity of the inference algorithm for the non-Markovian (model order
    $\order=T$) linear Gaussian model~\eqref{eq:linear-model}.}
  \label{tab:complexity}
  \medskip
  \begin{tabular}{lccc}
    \toprule
    \textbf{Algorithm} & \textbf{Time} & \textbf{Memory} & \textbf{Data structure} \\
    \midrule
    Recursive Kalman filter & $O(T^3 d^2)$ & $O(T^2 d^2)$ & Growing mean and covariance matrices \\
    Dual filter (non-recursive) & $O(T^2 d^2)$ & $O(T d)$ & Fixed $d\times T$ \\
    Transformer (per layer)     & $O(T^2 d^2)$  & $O(T d)$    & Fixed $d\times T$ \\
    \bottomrule
  \end{tabular}
  \medskip
\end{table}

%%%%%%%%%%%%%%%%%%%%%%%%%%%%%%%%%%%%%%%%%%%%%%%%%%%%%%%%%%%%

\section{Nonlinear Model}\label{sec:nonlinear}

Consider the graphical model in \Fig{fig:model} with discrete state-space $\bS=\{1,2,\hdots,d\}$ and observation-space $\bO=\{0,1,2,\hdots,m\}$ (see Table~\ref{tab:lin_non_lin_problem}).  
The next observation $Z_{t+1}$ depends on the past only through the current state $X_t$, with conditional distribution
\begin{subequations}\label{eq:HMM_model}
\begin{equation}\label{eq:HMM_C_model}
  \sP(Z_{t+1}=z | X_t=x) = C(x,z),\quad x\in\bS,\quad z\in\bO,\quad 0\leq t \leq T-1,
\end{equation}
where $C \in \Re^{d\times (m+1)}$ is a row stochastic matrix.  
The column vector $C(\cdot,z)$ is the probabilistic analogue of the embedding vector for token $z$ in a transformer.

The prediction target is the  conditional probability $\pi_T(x):=\sP(X_T = x \mid Z_{1:T})$ for  $x\in\bS$.  
The next-token prediction is then readily computed as
\begin{equation*}
  \sP(Z_{T+1}=z \mid Z_{1:T}) = \pi_T(C(\cdot,z)) := \sum_{x\in\bS} 
  \pi_T(x)\, C(x,z), \quad z\in\bO.
  \label{eq:next-token-prediction}
\end{equation*}
This computation is the probabilistic analogue of the un-embedding operation in a transformer.  

The main result of this section is the nonlinear predictor
representation for $\pi_T$, formally introduced in Table~\ref{tab:lin_non_lin_solution}, whose well-posedness is established in \Prop{prop:existence_nonlin_predictor_rep} below.  
The representation involves a mapping $e:\bO\to\Re^m$ defined as 
follows:
\begin{equation*}
  e(1) = \begin{bmatrix} 1 \\ 0 \\ \vdots \\ 0 \end{bmatrix}_{m\times 1}, e(2)
= \begin{bmatrix} 0 \\ 1\\ \vdots\\ 0 \end{bmatrix}_{m\times 1}, \hdots
\quad, e(m)
= \begin{bmatrix} 0 \\ 0 \\ \vdots \\ 1 \end{bmatrix}_{m\times 1},  e(0) =
-e(1) -e(2) - \hdots -e(m).
\end{equation*}
Here $e(z)$ for $z=1,2,\hdots,m$ is the one-hot encoding, and $e(0)$ is chosen so that $\sum_{z\in\bO} e(z) = 0$. 

\begin{example}[m=1]
  Suppose the observations are binary-valued, i.e., $\bO =
  \{0,1\}$.  Then
  \[
e(1) = 1,\quad e(0) = -1.
\]
\end{example}

In the nonlinear setting, the space of admissible weight sequences is
\[
\clU := \{U_t \in \Re^m : U_t \text{ is measurable w.r.t. } 
\clZ_t,\; 0\leq t\leq T-1\},
\]
where $\clZ_t := \sigma(\{Z_s : 1\leq s\leq t\})$ is the sigma-algebra generated by observations up to time $t$ with $\clZ_0$ the trivial sigma-algebra.  
In words, $U_t$ is required to depend only on past observations $Z_1, Z_2, \hdots, Z_t$ --- a causal dependence condition analogous to the causal masking in a decoder-only transformer. 
Note the contrast with $\clU^{\text{(det)}}$: 
the weights are now random (data-dependent) rather than deterministic, which is what makes the predictor nonlinear.

\begin{proposition}[Nonlinear predictor representation]
\label{prop:existence_nonlin_predictor_rep}
Fix $f:\bS\to\Re$. There exists $U\in\clU$ such that
\begin{equation*}
\label{eq:nonlin_predictor_rep}
\pi_T(f) = \text{(const.)} - \sum_{t=1}^{T} U_{t-1}^\transpose e(Z_{t}),
\quad \sP\text{-a.s.}.
\end{equation*}
If $\sP(Z_1=z_1,\hdots,Z_T=z_T)>0$ for all $(z_1,\hdots,z_T)\in\bO^T$, 
then $U$ is unique.
\end{proposition}
\begin{proof}
See \Sec{appx:well-posedness-nonlin-predictor-rep} in the appendix.
\end{proof}

The nonlinear predictor representation in~\Prop{prop:existence_nonlin_predictor_rep} reduces the prediction problem to the synthesis of the weight sequence $U\in\clU$. 
As in the linear Gaussian model, this synthesis is formulated as an optimal control problem. 
In the following, the results are described for the Markovian case ($\tau=1$), i.e., when $(X,Z)$ is an HMM. 
Unlike the linear Gaussian model where the OCP is solved for the general non-Markovian model, an explicit formula for $U$ is obtained here only for $\tau=1$.

\subsection{Optimal control problem (OCP) for computing weights}
\label{sec:docp_nonlinear}

The remainder of this section is closely adapted from~\cite{chang2025dual}. 
We make the following assumption: 
\begin{assumption}[Hidden Markov Model]
The pair $(X,Z)$ is a hidden Markov model (HMM) with emission matrix $C$ as defined in~\eqref{eq:HMM_C_model}, and state transition:
\begin{equation}\label{eq:HMM_A_model}
\sP(X_0=x) = \mu(x),~\sP(X_{t+1}=x' \mid X_t=x) = A(x,x'), \quad x,x'\in\bS,\quad 
t=0,1,\hdots,T-1,
\end{equation}
where $A\in\Re^{d\times d}$ is a row stochastic matrix.
\end{assumption}
\end{subequations}

The optimal control formulation for the HMM parallels that of the linear Gaussian model in \Sec{sec:linear}, with the  key objects summarized in Table~\ref{tab:nonlinear-objects}.  
For the HMM, the dual control system is a backward stochastic
difference equation (BS$\Delta$E) whose solution is denoted by $(Y,V)$ where $Y$ is a stochastic process whose linear analogue is the deterministic dual process $y$.  
The auxiliary process $V$ has no analogue in the linear case; the process enforces adaptedness of $Y$ to the filtration $\clZ_t$.  
The explicit formulae for these and the optimal control objective
$\sJ_T(U;f)$ appear in \Sec{appx:nonlinear} of the appendix.

\begin{table}[t]
  \centering
  \caption{Key objects in the nonlinear OCP and their linear 
    Gaussian analogues.  }
  \label{tab:nonlinear-objects}
  \medskip
  \begin{tabular}{cccc}
    \toprule
    \textbf{Object} & \textbf{Definition} & \textbf{Role} & 
    \textbf{Linear analogue} \\
    \midrule
    % $\clU$ & Causally adapted sequences & Admissible weights & 
    % $\clU^{\text{(det)}}$ \\[1ex]
    $U\in\clU$ & $\Re^m$-valued adapted process & Control/weights & 
                                                                  $u\in\clU^{\text{(det)}}$ \\[1ex]
        $f:\bS\to\Re$ & Function on $\bS$ & Prediction target is $\pi_T(f)$ & 
    $f\in\Re^d$ \\[1ex]
    $(Y,V)$ & Solution of BS$\Delta$E & Dual process & 
                                                       $y$ (backward \\[1ex]
                    & (with $U$, $Y_T=f$) & & process)
   \\[1ex]
    $\sJ_T(U;f)$ & Optimal control objective &  MSE of estimator $S_T$ & 
    $\sJ_T(u;f)$ \\
    \bottomrule
  \end{tabular}
  \smallskip
\end{table}

\begin{proposition}[{Duality principle for HMM~\cite[{Theorem 9}]{chang2025dual}}]
\label{prop:duality-nonlinear}
Consider an estimator
\begin{equation*}
S_T := \mu(Y_0) - \sum_{t=0}^{T-1} U_t^\tp e(Z_{t+1}),
\label{eq:nonlinear-estimator}
\end{equation*}
where $Y_0$ is obtained from the BS$\Delta$E~\eqref{eq:dual_BSDE} with $U\in\clU$ and $Y_T=f$. Then
\begin{equation*}
\sJ_T(U;f) = \E\left(\abs{f(X_T) - S_T}^2\right).
\label{eq:nonlinear-duality}
\end{equation*}
Consequently, minimizing $\sfJ_T(U;f)$ over $U\in\clU$ yields 
the MMSE estimator $\pi_T(f)$.
\end{proposition}

The duality principle reduces weight synthesis to the following optimal control problem:
\begin{equation}
\text{(OCP)} \qquad \min_{U\in\clU} \sfJ_T(U;f) \quad 
\text{subject to BS$\Delta$E \eqref{eq:dual_BSDE} in \Sec{appx:nonlinear}}.
\label{eq:nonlinear-OCP}
\end{equation}
Denote the conditional probability of the 
hidden state $X_t$ at time $t$ as $\pi_t(x):=\sP(X_t=x\mid Z_{1:t})$, 
$x\in\bS$, for $1\leq t\leq T$.  The solution of~\eqref{eq:nonlinear-OCP} yields a fixed-point 
representation of the conditional probability process 
$\pi=\{\pi_t:1\leq t\leq T\}$, as described in the following theorem. 

\begin{theorem}[{Formula for optimal control~\cite[{Theorem 11}]{chang2025dual}}]\label{thm:optimal-solution}
	Consider~\eqref{eq:nonlinear-OCP}.  
  Then an optimal control $U^\opt =\{U_t^\opt:0\leq t\leq T-1\}$ is of the feedback form given by
  \begin{subequations}
	  \begin{equation}\label{eq:optimal_control_formula}
  \text{(optimal control law)} \quad     U_t^\opt = \phi(Y_t,V_t;\pi_t),\quad \sP\text{-a.s.},\quad 0\leq t\leq T-1.
\end{equation}
(See equation \eqref{eq:control-nonlinear} in \Sec{appx:nonlinear} for an explicit form for $\phi:\Re^d\times\Re^{m\times d}\times\Re^d
\to\Re^m$).  Suppose 
$(Y^\opt,V^\opt)$ is the solution of the BS$\Delta$E with $U=U^\opt$. Then 
		\begin{align}
   %               \text{(fixed-point representation)} \quad
                  \pi_s (Y_s^\opt) = \mu (Y_0^\opt) - \sum_{t=1}^{s} (U_{t-1}^\opt)^\tp e(Z_{t}),\quad \sP\text{-a.s.},\quad 1\leq s\leq T.
      \label{eq:estimator-t}
		\end{align}
	\end{subequations}
      \end{theorem}

\Thm{thm:optimal-solution} justifies the layer 
transformation~\eqref{eq:layer_dual_filter} introduced in 
\Sec{sec:intro}.  Eq.~\eqref{eq:estimator-t} shows that the
optimal control $U^\opt$ yields a causal representation of the
conditional probability process $\pi$: At a query time-step $s$, the
prediction $\pi_s$ is computed in terms of the weights
$\{U_{t-1}^\opt:1\leq t\leq s\}$ (see \Fig{fig:control-attention}).  Moreover, since $U_t^\opt =
\phi(Y_t,V_t;\pi_t)$ depends on $\pi_t$
through~\eqref{eq:optimal_control_formula},
equation~\eqref{eq:estimator-t} is a fixed-point representation of
$\pi$.  This fixed-point structure defines a natural layer 
transformation: given an approximation $\rho\approx\pi$ as input, 
one pass through the BS$\Delta$E, using the control 
law~\eqref{eq:optimal_control_formula} with $\rho_t$ instead of
$\pi_t$, produces an updated approximation $\rho^+$ as output.

\vspace*{-0.1in}
\paragraph{Related work.}
Relative to~\cite{chang2025dual}, the novel contributions of this
paper are: (i) the nonlinear predictor representation
(\Prop{prop:existence_nonlin_predictor_rep}), which is
proved here for the general non-Markovian model; and (ii) the parallel treatment of the nonlinear and linear Gaussian
models within a unified optimal control framework. The explicit formula for $U$ is given only for the Markovian 
case ($\tau=1$), based on the original derivation of the same in~\cite{chang2025dual}. The numerical experiments in \Sec{sec:numerics} 
suggest that extending the framework to the non-Markovian case 
($\tau > 1$) is both important and feasible --- a transformer implicitly solves this problem through attention, and the dual filter framework provides a mathematical foundation for a principled extension.

\section{Numerical Experiments}
\label{sec:numerics}

The observation process $Z=\{Z_t:1\le t\le T\}$ takes values in $\bO=\{0,1\}$ (i.e. binary observations $m=1$) and is constructed by concatenating two types of deterministic patterns:
\[
z^{\text{long}}=\{1,1,\underbrace{0,\ldots,0}_{(d-2) \; \text{zeros}}\}, 
\qquad
z^{\text{short}}=\{1,\underbrace{0,\ldots,0}_{q\text{ zeros}}\},
\]
where $d$ is the length of the long pattern and $q<d-2$ is the number
of zeros in the short pattern.  At the end of each cycle, the next
pattern type (long or short) is selected independently with
probability $0.5$, and the corresponding pattern is appended.

The stochastic process $Z$ is a $d$-th order Markov process: 
The conditional probability of $Z_t$ depends on the position within the current cycle, which in the worst case (a long cycle) requires knowledge of up to the past $d$ observations.  
A minimal HMM representation on state-space $\bS=\{1,2,\hdots,d\}$ is
depicted in \Fig{fig:model-observation} together with a sample path of
$Z$ showing the two types of patterns.

\begin{remark}
The choice of two patterns is deliberate: all predictive information
is carried by the $1$s, which are sparse and always surrounded by 
a long string of zeros. This raises two questions: whether a fully 
trained nanoGPT focuses its attention weights on the $1$s, and 
how those attention weights compare structurally to the optimal 
control weights.
\end{remark}

We have three goals in this study:
\begin{enumerate}
\item Evaluate the attention weights using a decoder-only transformer
  (nanoGPT) with its embedding dimension set to $d$.
  
  \item Evaluate the optimal control weights using the dual filter 
with the true HMM parameters, and compare their structure to 
the attention weights.

\item Evaluate the effect of the candidate state dimension $\hat{d}$ 
  on inference performance. The HMM parameters are learned via 
  Baum-Welch for each $\hat{d}$, covering the undercomplete 
  ($\hat{d} < d$), matched ($\hat{d} = d$), and overcomplete 
  ($\hat{d} > d$) regimes.
  
\end{enumerate}

In the numerical experiments reported below, we fix $d=16$, $q=4$, and
$T=64$. (Results are qualitatively similar for other values of $d$ and 
$T$, as described in \Sec{appdx:high_dim}.)  The 
performance is assessed using the cross-entropy loss of next-token
prediction. Because the data is generated from a known HMM, the optimal 
(minimum) loss equals the conditional entropy under the true 
nonlinear filter, which serves as the benchmark throughout.

\subsection{Comparison of attention and control weights}

% A nanoGPT \cite{karpathy2024nanogpt} with a single attention head, single layer, and the embedding dimension set to $d$ achieves nearly the optimal loss (see
% \Fig{fig:loss} in Appendix \ref{appx:numerics} for the loss over iterations and a discussion of learning settings.)   

A nanoGPT \cite{karpathy2024nanogpt} with one layer, one attention head, and embedding dimension $d$ achieves near-optimal loss (see Appendix \ref{appx:numerics} for training details).

In the model-based setting, the optimal control weights are computed
using the formula derived in \Sec{sec:nonlinear}.
\Fig{fig:control-attention} depicts a side-by-side comparison of the attention weights and
the control weights.  
Qualitatively, both sets of weights show a sparse pattern which is non-zero at the time indices when $1$s are seen in the data.
There is a simple explanation for sparsity and structure of the control weights.  
The control is non-zero only following those time-indices $t$ such that
$X_t=d$ with associated emission $Z_{t+1}=1$.   
At the branching point, the filter splits equally between the two
possible next states: the filter $\pi_{t+1}(x) =
\frac{1}{2}(\delta_{1}(x) + \delta_{d-q}(x))$, yielding a non-zero value $U_{t+1}^\opt$.  
At all other time-indices, the filter has its probability mass only on one state with associated control as 0.      

These results demonstrate that the transformer has learned,
without any knowledge of the HMM, a weighting structure that
qualitatively resembles the optimal control weights derived from first
principles.   
The more interesting regime is when the embedding dimension is smaller than the true (Markovian) state dimension.  
This is the subject of the following experiment.

\begin{figure}[t]
  \vspace*{-1.25\baselineskip}
  \centering
  \includegraphics[width=0.95\linewidth]{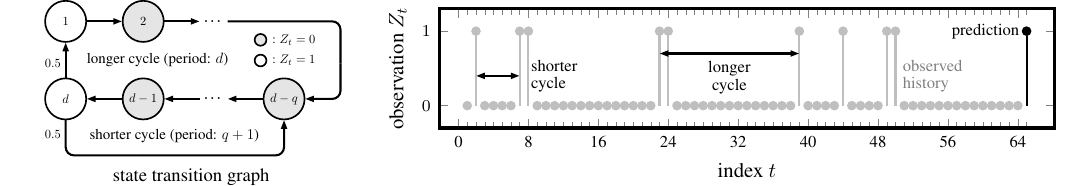}
  % \vspace*{1pt}
  \caption{Two-cycle HMM and a sample observation sequence.
    (left) State transition graph: state $d$ branches 
    with equal probability into a long cycle or a 
    short cycle, merging at state $d-q$ before 
    returning to state $d$. States $1$ and $d$ emit $Z_t=1$; 
    all others emit $Z_t=0$.
    (right) Sample trajectory with observed history indicated as gray and the prediction target as black.
  }
  \label{fig:model-observation}
\end{figure}

\begin{figure}[t]
  \vspace*{-0.75\baselineskip}
  \centering
  \includegraphics[width=0.9\linewidth]{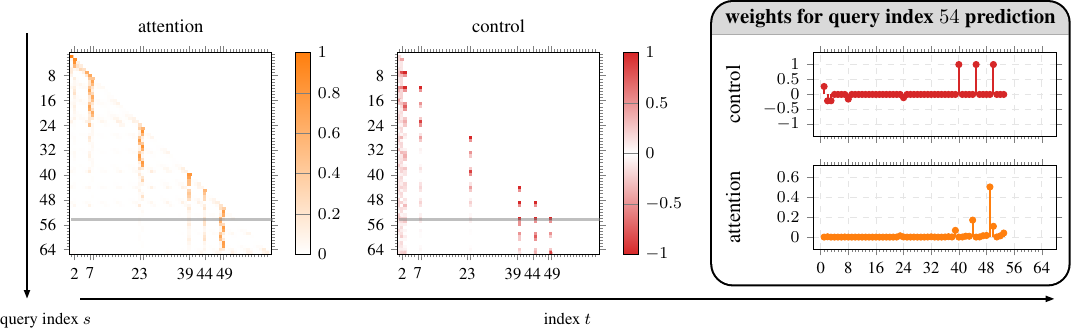}
\caption{%
    Attention and control weight heatmaps for $d=\hat{d}=16$ and $T=64$. 
    (left) Attention matrix from a trained nanoGPT: sparse, concentrated at time indices where $Z_t=1$.
    (middle) Dual filter control weights computed with exact HMM parameters: same sparsity pattern, non-zero only where $Z_t=1$.
    (right) Slice at the query step $s=54$: control (top) and attention (bottom) both focus on time indices where $Z_t=1$, for $1\leq t\leq s$.
    The corresponding query step is highlighted in the attention and control heatmaps with a horizontal gray background.
    The major ticks in the index axis of the heatmaps correspond to the time indices where $Z_t=1$; the minor ticks correspond to the time indices where $Z_t=0$.
}
  \label{fig:control-attention}
  \vspace*{-0.5\baselineskip}
\end{figure}

\subsection{Effect of dimension $\hat{d}$ (non-Markovian advantage)}

\begin{figure}[t]
  \vspace*{-0.25\baselineskip}
  \centering
  \includegraphics[width=0.9\linewidth]{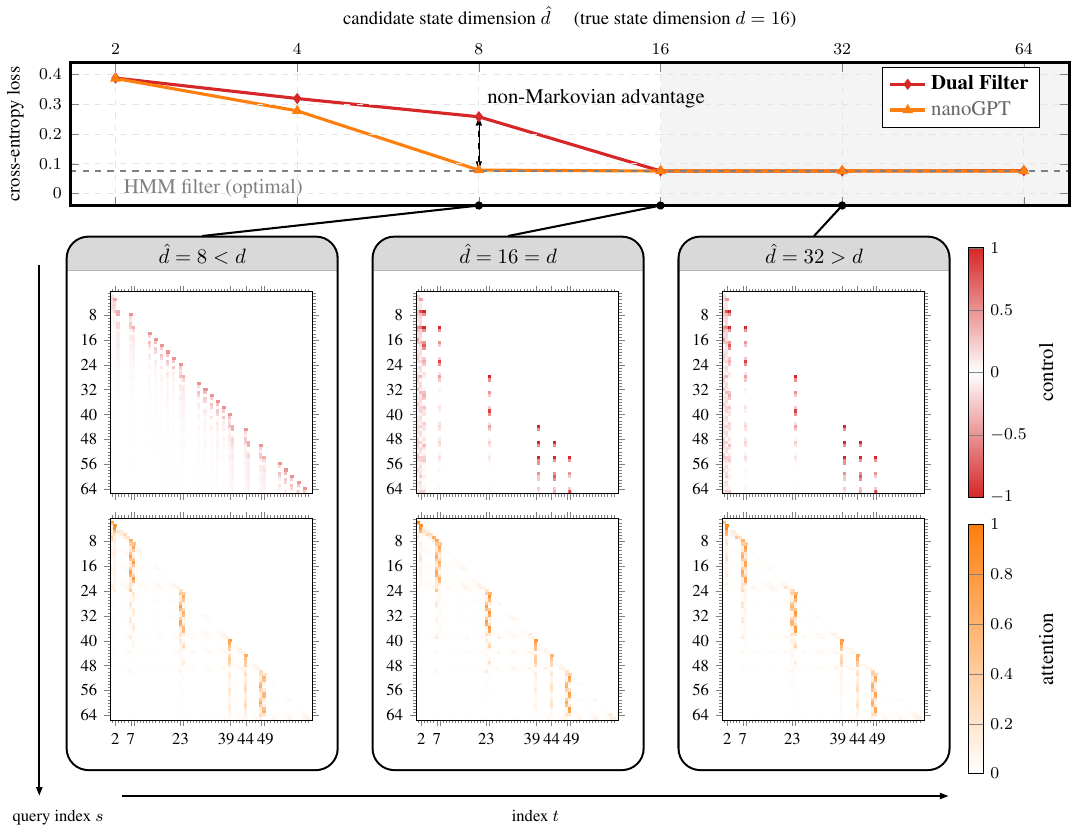}
 \caption{Non-Markovian advantage.
  (top) Cross-entropy loss vs. $\hat{d}$ for the dual filter and nanoGPT; dashed line is the HMM filter.
  (bottom) Control and attention heatmaps at $\hat{d} \in \{8, 16, 32\}$; color intensity indicates weight magnitude.
  The major ticks in the index axis correspond to the time indices where $Z_t=1$; the minor ticks correspond to the time indices where $Z_t=0$.
  When $\hat{d} < d$, controls degrade while attentions remain focused on $Z_t=1$, revealing the non-Markovian advantage.
}
  \label{fig:nma}
\end{figure}

In this experiment, the performance of the dual filter and nanoGPT are
evaluated as the candidate state dimension $\hat{d}$ is varied. 
The true state dimension is fixed at $d=16$, and $\hat{d}$ is tested
both below and above $d$.  
For each $\hat{d}$, including for $\hat{d}=d$, the HMM parameters $(A,C)$ are estimated from data using the Baum-Welch algorithm; 
the same dataset is used to train nanoGPT.

% For $\hat{d} \geq d$, both methods achieve nearly optimal performance. 
% The interesting regime is $\hat{d} < d$: the dual filter's Markovian 
% model (for the hidden state) cannot capture the full temporal structure of the observation 
% sequence, leading to significant performance degradation. 
% In contrast, nanoGPT retains near-optimal performance even at 
% $\hat{d}=\nicefrac{d}{2}$, implicitly exploiting non-Markovian 
% structure through its use of embedding and the attention mechanism.

For $\hat{d} \geq d$, both methods achieve near-optimal performance. 
The more interesting regime is $\hat{d} < d$, where the dual filter's Markovian model fails to capture the full temporal structure of the observations, resulting in performance degradation. 
In contrast, nanoGPT maintains near-optimal performance even at $\hat{d}=\nicefrac{d}{2}$, implicitly leveraging non-Markovian structure through its attention mechanism.

\Fig{fig:nma} shows the attention weights at three representative values: $\hat{d}\in\{8,16,32\}$.
Across all three values, the attention weights consistently focus on time indices where $Z_t=1$, identifying the informative observations regardless of model misspecification.
This gap --- the dual filter degrades while nanoGPT does not --- is precisely the \emph{non-Markovian advantage}~\cite{chen2026differentiable}: the transformer implicitly operates beyond the Markovian regime, motivating the extension of the dual filter to the non-Markovian nonlinear case as a direction for future work.

\subsection{Sparsity of attention and control weights}

% Transformer architectures are designed to produce sparse attention 
% weights, but is sparsity optimal? In the two-cycle HMM, the 
% emission is deterministic.  The filter collapses to a point mass 
% except at the branching state, so the optimal control weights are 
% sparse with a vertical pattern --- non-zero only at time indices
% immediately following when $X_t = d$.

% A perturbation study (\Sec{appx:robust}) reveals a qualitative change
% in the optimal control weights: even with small perturbations in the
% HMM parameters $(A,C)$, they shift to a sparse but diagonal pattern, aggregating 
% evidence across multiple past cycles rather than relying on a single branching moment (\Fig{fig:robust}). In contrast, 
% the nanoGPT attention weights remain sparse and vertical throughout --- a 
% fixed pattern that nonetheless tracks the optimal loss of the dual filter 
% closely across perturbations.

Transformer architectures are designed to produce sparse attention weights, but whether sparsity is optimal remains unclear. 
In the two-cycle HMM, deterministic emissions cause the filter to collapse to a point mass except at the branching state, yielding sparse optimal control weights with a vertical pattern — nonzero only immediately after $X_t=d$.

A perturbation study reveals a qualitative shift in the optimal
control structure (\Sec{appx:robust} in Appendix). 
Even small perturbations to the transition parameter $A$ induce a
sparse but diagonal pattern, aggregating evidence across multiple past
cycles rather than relying on a single branching moment
(\Fig{fig:robust} in \Sec{appx:robust}).
In contrast, the nanoGPT attention weights remain sparse and vertical across perturbations, while still closely matching the optimal loss of the dual filter.

We caution against over-interpreting weight structure: performance and weight pattern are distinct.
In the non-Markovian advantage of \Fig{fig:nma}, the dual filter's degradation in the undercomplete regime is a capacity issue, not a consequence of the attention weight structure.  
On the other hand, \Fig{fig:robust} shows that the transformer achieves near-optimal loss even in the perturbed case. 
Whether sparse attention is what enables nanoGPT to robustly exploit non-Markovian structure remains an open question.

\section{Conclusion}
\label{sec:conc}

This paper derives transformer-like inference architectures from 
first principles using optimal control theory. For two model 
classes --- a linear Gaussian model and a nonlinear discrete-valued 
model --- the prediction objective is reformulated as an optimal 
control problem whose solution defines a $d\times T$ 
sequence-to-sequence transformation, structurally analogous to a 
transformer layer, emerging from optimality conditions rather than 
specified by architecture.

Numerical experiments provide a direct comparison of optimal 
control weights and attention weights from a trained nanoGPT. 
In the matched regime ($\hat{d}=d$), the two weight structures 
are qualitatively aligned --- both sparse and focused on the 
informative observations. 
In the undercomplete regime ($\hat{d} < d$), the dual 
filter degrades while nanoGPT retains near-optimal performance 
--- the non-Markovian advantage --- revealing that the transformer 
implicitly exploits non-Markovian structure that the Markovian 
dual filter cannot represent.  A perturbation study
(\Sec{appx:robust}) further shows that the sparsity of the optimal
control weights is model-dependent rather than algorithmic. 

Several limitations point to directions for future work. First, the explicit optimal 
control formula is obtained only for the Markovian case ($\tau=1$) 
in the nonlinear model. Second, the framework is entirely model-based: 
the parameters $(A, C, \mu)$ are assumed known, and developing 
a learning framework to estimate either these parameters or control
weights directly from data is ongoing work.  

The non-Markovian advantage identifies a precise gap: the Markovian
dual filter is insufficient, and a non-Markovian nonlinear extension
is needed --- the linear Gaussian results of \Sec{sec:linear} and the
dual filter framework provide the mathematical foundation for this
next step. 

% This paper develops a duality-based optimal control framework for understanding and designing predictors in non-Markovian models, with a particular focus on the linear Gaussian and nonlinear HMM cases. 
% The key insight is that the prediction problem can be reformulated as an optimal control problem, where the control inputs correspond to the weights in a predictor representation. 
% This framework not only provides a theoretical foundation for yielding practical algorithms, such as the dual filter, but also offers a lens through which to interpret the attention mechanisms in transformers as implicitly learning an optimal control strategy for inference.
% The numerical experiments on a synthetic two-cycle HMM illustrate the alignment between the transformer's attention and the dual filter's optimal control, as well as the robustness of the transformer's attention to model misspecification, demonstrating the practical relevance of the theoretical insights. 
% Overall, this work bridges the gap between optimal control theory and modern sequence modeling, providing a first principles approach to designing predictors in complex, non-Markovian settings.

\bibliographystyle{unsrt}
\bibliography{references}

@book{kailath2000linear,
  title={Linear estimation},
  author={Kailath, Thomas and Sayed, Ali H and Hassibi, Babak},
  year={2000},
  publisher={Prentice Hall}
}

@article{chang2025dual,
  title={Dual filter: A mathematical framework for inference using transformer-like architectures},
  author={Chang, Heng-Sheng and Mehta, Prashant G},
  journal={arXiv preprint arXiv:2505.00818},
  year={2025}
}

@article{phuong2022formal,
  title={Formal algorithms for transformers},
  author={Phuong, Mary and Hutter, Marcus},
  journal={arXiv preprint arXiv:2207.09238},
  year={2022}
}

@article{akman2026optimal,
  title={An Optimal Control Approach To Transformer Training},
  author={Akman, Ka{\u{g}}an and Sald{\i}, Naci and Y{\"u}ksel, Serdar},
  journal={arXiv preprint arXiv:2603.09571},
  year={2026}
}

@inproceedings{kan2025optimal,
  title={Optimal Control for Transformer Architectures: Enhancing Generalization, Robustness and Efficiency},
  author={Kan, Kelvin and Li, Xingjian and Zhang, Benjamin and Sahai, Tuhin and Osher, Stanley and Katsoulakis, Markos},
  booktitle={The Thirty-ninth Annual Conference on Neural Information Processing Systems}
}

@article{cimetiere2025localmax,
  title={Localmax dynamics for attention in transformers and its asymptotic behavior},
  author={Cimeti{\`e}re, Henri and Chiri, Maria Teresa and Gharesifard, Bahman},
  journal={arXiv preprint arXiv:2509.15958},
  year={2025}
}

@article{krishnamurthy2026llms,
  title={LLMs as High-Dimensional Nonlinear Autoregressive Models with Attention: Training, Alignment and Inference},
  author={Krishnamurthy, Vikram},
  journal={arXiv preprint arXiv:2602.00426},
  year={2026}
}

@inproceedings{goel2024can,
  title={Can a transformer represent a Kalman filter?},
  author={Goel, Gautam and Bartlett, Peter},
  booktitle={6th Annual Learning for Dynamics \& Control Conference},
  pages={1502--1512},
  year={2024},
  organization={PMLR}
}

@article{du2023can,
  title={Can transformers learn optimal filtering for unknown systems?},
  author={Du, Zhe and Balim, Haldun and Oymak, Samet and Ozay, Necmiye},
  journal={IEEE Control Systems Letters},
  volume={7},
  pages={3525--3530},
  year={2023},
  publisher={IEEE}
}

@article{geshkovski2024measure,
  title={Measure-to-measure interpolation using transformers},
  author={Geshkovski, Borjan and Rigollet, Philippe and Ruiz-Balet, Dom{\`e}nec},
  journal={arXiv preprint arXiv:2411.04551},
  year={2024}
}

@article{castin2025unified,
  title={A unified perspective on the dynamics of deep transformers},
  author={Castin, Val{\'e}rie and Ablin, Pierre and Carrillo, Jos{\'e} Antonio and Peyr{\'e}, Gabriel},
  journal={arXiv preprint arXiv:2501.18322},
  year={2025}
}

@article{xie2021explanation,
  title={An explanation of in-context learning as implicit bayesian inference},
  author={Xie, Sang Michael and Raghunathan, Aditi and Liang, Percy and Ma, Tengyu},
  journal={arXiv preprint arXiv:2111.02080},
  year={2021}
}

@article{bai2023transformers,
  title={Transformers as statisticians: Provable in-context learning with in-context algorithm selection},
  author={Bai, Yu and Chen, Fan and Wang, Huan and Xiong, Caiming and Mei, Song},
  journal={Advances in neural information processing systems},
  volume={36},
  pages={57125--57211},
  year={2023}
}

@article{alaa2019attentive,
  title={Attentive state-space modeling of disease progression},
  author={Alaa, Ahmed M and van der Schaar, Mihaela},
  journal={Advances in neural information processing systems},
  volume={32},
  year={2019}
}

@article{tang2021probabilistic,
  title={Probabilistic transformer for time series analysis},
  author={Tang, Binh and Matteson, David S},
  journal={Advances in neural information processing systems},
  volume={34},
  pages={23592--23608},
  year={2021}
}

@book{bar2001estimation,
  title={Estimation with applications to tracking and navigation: theory algorithms and software},
  author={Bar-Shalom, Yaakov and Li, X Rong and Kirubarajan, Thiagalingam},
  year={2001},
  publisher={John Wiley \& Sons}
}

@article{gu2023mamba,
  title={Mamba: Linear-time sequence modeling with selective state spaces},
  author={Gu, Albert and Dao, Tri},
  journal={arXiv preprint arXiv:2312.00752},
  year={2023}
}

@article{kim2025arrow,
  title={The arrow of time in estimation and control: Duality theory beyond the linear Gaussian model},
  author={Kim, Jin Won and Mehta, Prashant G},
  journal={IEEE Control Systems},
  volume={45},
  number={2},
  pages={70--90},
  year={2025},
  publisher={IEEE}
}

@inproceedings{kim2019duality,
	Author = {Kim, Jin Won and Mehta, Prashant G. and Meyn, Sean},
	Booktitle = {2019 IEEE 58th Conference on Decision and Control (CDC)},
	Title = {What is the {Lagrangian} for Nonlinear Filtering?},
	Year = {2019},
	pages = {1607--1614},
	month={12},
	address={Nice, France},
	organization={IEEE}
}

@article{kim2023duality,
  title={Duality for nonlinear filtering II: Optimal control},
  author={Kim, Jin Won and Mehta, Prashant G},
  journal={IEEE Transactions on Automatic Control},
  volume={69},
  number={2},
  pages={712--725},
  year={2023},
  publisher={IEEE}
}

@book{bensoussan2018estimation,
  title={Estimation and control of dynamical systems},
  author={Bensoussan, Alain},
  volume={48},
  year={2018},
  publisher={Springer}
}

@inproceedings{todorov2008general,
  title={General duality between optimal control and estimation},
  author={Todorov, Emanuel},
  booktitle={2008 47th IEEE conference on decision and control},
  pages={4286--4292},
  year={2008},
  organization={IEEE}
}

@software{karpathy2024nanogpt,
  title = {\verb|nanoGPT|: The simplest, fastest repository for training/finetuning medium-sized GPTs.},
  author = {Karpathy, Andrej},
  year = {2024},
  howpublished = {\url{https://github.com/karpathy/nanoGPT}},
}

@inproceedings{chen2026differentiable,
  title={Differentiable Filtering for Learning Hidden Markov Models},
  author={Reginald Zhiyan Chen and Heng-Sheng Chang and Prashant G Mehta},
  booktitle={8th Annual Learning for Dynamics and Control Conference},
  year={2026},
}

@article{abella2024asymptotic,
  title={The Asymptotic Behavior of Attention in Transformers},
  author={Abella, {\'A}lvaro Rodr{\'\i}guez and Silvestre, Jo{\~a}o Pedro and Tabuada, Paulo},
  journal={arXiv preprint arXiv:2412.02682},
  year={2024}
}

@article{adu2024approximate,
  title={Approximate controllability of continuity equation of transformers},
  author={Adu, Daniel Owusu and Gharesifard, Bahman},
  journal={IEEE Control Systems Letters},
  year={2024},
  publisher={IEEE}
}

@article{geshkovski2023mathematical,
  title={A mathematical perspective on transformers},
  author={Geshkovski, Borjan and Letrouit, Cyril and Polyanskiy, Yury and Rigollet, Philippe},
  journal={arXiv preprint arXiv:2312.10794},
  year={2023}
}

@article{rohekar2023causal,
  title={Causal interpretation of self-attention in pre-trained transformers},
  author={Rohekar, Raanan Y and Gurwicz, Yaniv and Nisimov, Shami},
  journal={Advances in Neural Information Processing Systems},
  volume={36},
  pages={31450--31465},
  year={2023}
}

% \end{document}

\appendix

\section{Explicit formulae for the dual optimal control problem in \Sec{sec:docp_nonlinear}}\label{appx:nonlinear}

This section provides explicit expressions for the BS$\Delta$E dual
control system, the
optimal control objective, and the formula for optimal control.
Additional details and background on these can be found
in~\cite{chang2025dual}.  We follow closely the notation of~\cite{chang2025dual}.  Denote
\begin{align*}
  c(x) := \begin{bmatrix} C(x, 1) - C(x, 0) & C(x, 2) - C(x, 0) &
    \hdots & C(x, m) - C(x, 0) \end{bmatrix}^\transpose, \; \ x \in \bS.
\end{align*}

\paragraph{Dual Control System.}~The dual control system is a backward stochastic difference equation
(BS$\Delta$E)  as follows:
\begin{subequations}\label{eq:dual_BSDE}
  \begin{align}
    Y_t(x) &= (AY_{t+1})(x) + c^\transpose(x) (U_t + V_t(x)) - V_t^\transpose(x)
            e(Z_{t+1}),~x\in\bS,~t=0,1,\hdots,T-1, \label{eq:dual_BSDE_a}\\
    Y_T(x)  & = f(x),\quad x\in\bS.  \label{eq:dual_BSDE_b}
  \end{align}
\end{subequations}
Here $f:\bS\to\Re$ is the terminal condition and $U:=\{U_t:0\leq
t\leq T-1\} \in \clU$ is the control input (these are both given).  The BS$\Delta$E is
solved to obtain the solution pair
\begin{align*}
  Y&:=\{Y_t(x): Y_t(x)\in\clZ_t,~x\in\bS,~0\leq t\leq T\},\\
  V&:=\{V_t(x): V_t(x)\in\clZ_t,~x\in\bS,~0\leq t\leq T-1\},
\end{align*}
where for each fixed $x\in\bS$, $Y_t(x)$ is real-valued and $V_t(x)$ is $\Re^m$-valued.
The solution pair is denoted by $(Y,V)$.  Its existence and uniqueness
is shown in~\citep[{Prop.~8}]{chang2025dual}. 
\paragraph{Optimal Control Objective.}~The optimal control objective
is defined as follows:
\begin{align*}
  \sJ_T(U;f)  := \dvar (Y_0(X_0)) + \E \Big( \sum_{t=0}^{T-1} l (Y_{t+1},V_t,U_t\,;X_t)  \Big),
  \label{eq:nonlinear-objective}
\end{align*}
where $\dvar(Y_0(X_0))=\E(|Y_0(X_0) - \mu(Y_0)|^2) =
\mu(Y_0^2)-\mu(Y_0)^2$ (note that $Y_0$ is a deterministic function),
and the running cost $l:\Re^d\times\Re^{m\times
  d}\times\Re^m\times\bS\to\Re$ is given by, 
\begin{equation*}\label{eq:running_cost_formula}
  l(y,v,u;x):= (\Gamma y)(x) + (u+v(x))^\transpose R(x) (u+v(x)),\quad y\in\Re^d,\;v\in\Re^{m\times d},\;u\in\Re^m,\;x\in\bS.
\end{equation*}
where  \begin{align*}
    (\Gamma f)(x) &:= \sum_{y\in\bS} A(x,y) f^2(y) - (Af)^2(x),\quad x\in\bS, \\
    R(x) &:=  \text{diag}(c(x)) + C(x,0) (I+\ones \ones^\transpose)  - c(x) c^\transpose (x), \quad x\in\bS,
  \end{align*}
and $v(x)$ is the $x$-th column vector of the $m\times d$ matrix
$v=\begin{bmatrix}v(1) & \hdots & v(x) & \hdots &
  v(d) \end{bmatrix}_{m\times d}$.

\paragraph{Optimal control formula.}  The explicit expression for the
function $\phi$ introduced in~\eqref{eq:optimal_control_formula} is 
\begin{equation*}
  \phi(y,v;\rho) := - \rho(R)^{\dagger} \left( \rho ((c-\rho (c))y)- \rho (Rv)\right), \quad y\in\Re^d, \; v\in \Re^{m\times d},\; \rho\in\clP(\bS).
  \label{eq:control-nonlinear}
\end{equation*}
Here, $\rho(R)^{\dagger}$ denotes the pseudo-inverse of the $m\times m$ matrix $\rho(R):=\sum_{x\in\bS} \rho(x) R(x)$ and the other two terms are $\rho ((c-\rho (c))y) := \sum_{x\in\bS} \rho(x) (c(x) - \rho(c)) y(x)$ and $\rho (Rv):= \sum_{x\in\bS} \rho(x) R(x)v(x)$ which are $m\times 1$ vectors.

The formula is used to compute the attention weights described in the
numerical experiments.  

\section{Proofs}

\subsection{Proof of the \Prop{prop:duality_LG} (duality principle for non-Markovian linear Gaussian model)}
\label{appx:proof-duality-linear}

First, using the observation process \eqref{eq:linear-observation}, state process equation \eqref{eq:linear-process} and the backward process (dual control system) \eqref{eq:linear-dual-backward}, we compute
\begin{equation*}
  \dualstate_\horizon^\transpose\stateProcess_\horizon = \dualstate_\initial^\transpose\stateProcess_\initial - \sum_{t=0}^{\horizon-1}\control_t^\transpose\observationProcess_{t+1} + \sum_{t=0}^{\horizon-1}\control_t^\transpose\observationNoise_{t+1} + \sum_{t=1}^{\horizon}\dualstate_{t}^\transpose\processNoise_t
\end{equation*}

Now, noting the terminal condition of the dual state $\dualstate_\horizon=f$, and the definition of estimator $S_\horizon$ in \eqref{eq:linear-estimator}, we get
\begin{equation}\label{eq:pairing-identity}
  f^\transpose \stateProcess_{\horizon} - S_\horizon = \dualstate_\initial^\transpose(\stateProcess_\initial - \initialMean) + \sum_{t=1}^{\horizon}\dualstate_{t}^\transpose\processNoise_{t} + \sum_{t=0}^{\horizon-1}\control_t^\transpose\observationNoise_{t+1}
\end{equation}

Next, we square this equation and take expectations. Since $\stateProcess_\initial-\initialMean\sim\Gaussian(0,\initialCovariance)$, $\processNoise_t\sim\Gaussian(0,\covarianceProcessNoise)$, $\observationNoise_t\sim\Gaussian(0,\covarianceObservationNoise)$, and $\stateProcess_\initial,\processNoise,\observationNoise$ are mutually independent, all three terms on the right hand side of~\eqref{eq:pairing-identity} are zero-mean and mutually uncorrelated.
Therefore,
\begin{equation*}
  \Expectation{\abs{f^\transpose\stateProcess_{\horizon}-S_\horizon}^2} = \abs{\dualstate_\initial}^2_{\initialCovariance} + \sum_{t=1}^{\horizon}\abs{\dualstate_{t}}^2_{\covarianceProcessNoise} + \sum_{t=0}^{\horizon-1}\abs{\control_t}^2_{\covarianceObservationNoise} = 2 \sJ_T(u;f)
\end{equation*}
and this concludes the proof.
\hfill$\square$

\subsection{Proof of \Thm{Thm:OCP_LG} (solution of the linear optimal control problem)}
\label{appx:proof-optimal-control-linear}

Let $\control^{\text{(opt)}}$ denote the optimal control and consider an arbitrary perturbation $\Delta\control$. The perturbed control trajectory can be written as $\control = \control^{\text{(opt)}} + \Delta\control$. Since the dual control system~\eqref{eq:linear-dual-backward} is linear, the resulting dual state $\dualstate$ will be of the form $\dualstate = \dualstate^{\text{(opt)}} + \Delta\dualstate$. Here, $\dualstate^{\text{(opt)}}$ and $\Delta\dualstate$ are the unique solutions to the dual dynamics corresponding to $\control^{\text{(opt)}}$ (with terminal condition $\dualstate^{\text{(opt)}}_\horizon = f$) and $\Delta\control$ (with terminal condition $\Delta\dualstate_\horizon = 0$), respectively. Now, the cost functional $\dualCost_\horizon(\control; f)$ can be expanded around $\control^{\text{(opt)}}$ as follows:
\begin{equation*}
  \dualCost_\horizon(\control; f) = \dualCost_\horizon(\control^{\text{(opt)}}; f) + \mathcal{I}_\horizon(\control^{\text{(opt)}}, \Delta\control) + \dualCost_\horizon(\Delta\control; 0)
\end{equation*}

where $\dualCost_\horizon(\control^{\text{(opt)}}; f)$ represents the minimum cost, $\dualCost_\horizon(\Delta\control; 0)$ is the second-order error term (always non-negative due to convexity), and the first-order variation (cross terms) is given by:
\begin{align*}
  \mathcal{I}_\horizon(\control^{\text{(opt)}}, \Delta\control) &= (\dualstate_\initial^{\text{(opt)}})^\transpose\initialCovariance(\Delta\dualstate_\initial) + \sum_{t=1}^{\horizon}(\dualstate_{t}^{\text{(opt)}})^\transpose\covarianceProcessNoise(\Delta\dualstate_{t}) + \sum_{t=\initial}^{\horizon-1}(\control_t^{\text{(opt)}})^\transpose\covarianceObservationNoise(\Delta\control_t)
\end{align*}

Using the definition of the forward momentum process $\eta$ from~\eqref{eq:linear-fb-forward} and the dual control system dynamics for $\Delta\dualstate$, we can simplify the first-order variation into the following inner product form:
\begin{equation*}
  \mathcal{I}_\horizon(\control^{\text{(opt)}}, \Delta\control) = \sum_{t=\initial}^{\horizon-1} \left( \observationModel \eta_t + \covarianceObservationNoise \control^{\text{(opt)}}_{t} \right)^\transpose (\Delta\control_t)
\end{equation*}

The optimality of $\control^{\text{(opt)}}$ requires that the first-order variation $\mathcal{I}_\horizon(\control^{\text{(opt)}}, \Delta\control)$ must vanish for any arbitrary perturbation $\Delta\control$. This implies $\control^{\text{(opt)}}_{t} = - \covarianceObservationNoise^{-1} \observationModel \eta_t$ for all $t \in \set{\initial, \dots, \horizon-1}$, which gives us the optimal control formula~\eqref{eq:linear-fb-control}. Also, the strict convexity of the cost functional with respect to $\control$ ensures that this stationary point is the unique global minimizer. This completes the proof. \hfill$\square$

\subsection{Proof of \Prop{prop:existence_nonlin_predictor_rep} (well-posedness of the nonlinear predictor representation)}
\label{appx:well-posedness-nonlin-predictor-rep}

The existence theorem relies on the following proposition from
linear algebra.

\begin{proposition}[{See~\cite[{Prop.~28}]{chang2025dual}}]\label{prop:linear_algebra}
  Let $s:\bO\to \Re$. Then there exists unique $(s,\tilde{s}) \in \Re
  \times \Re^m$ such that the following decomposition holds:
  \[
  s(z) = \bar{s} +\tilde{s}^\tp e(z) ,\quad z\in\bO.
\]
Explicitly, 
\[
\bar{s}:= \frac{1}{m+1} \sum_{z\in\bO} s(z),\quad \text{and} \quad 
\tilde{s} (i) = (s(i) - \bar{s}),\quad
i=1,2,\hdots,m.
\]
\end{proposition}

\begin{example}[m=1] Let $s:\{0,1\} \to \Re$.  Denote
$
   s^+ := s(1)$ and $s^{-} := s(0)
 $. Then
 \[
 s(z) = \bar{s} +  \tilde{s} \, e(z),\quad z\in\{0,1\},
\]
where $\bar{s}:=0.5 (s^+ + s^-)$ and $\tilde{s}:= 0.5 (s^+ -
  s^-)$ (recall $e(1)=1$ and $e(0)=-1$).
\end{example}

We now provide a proof of 
\Prop{prop:existence_nonlin_predictor_rep}, establishing 
well-posedness of the nonlinear predictor representation in~\Prop{eq:nonlin_predictor_rep}. 
The proof follows the analogous result for the HMM 
in~\cite{chang2025dual}.

\begin{proof}[Proof of \Prop{prop:existence_nonlin_predictor_rep}]
Let $S_T$  be any $\clZ_T$-measurable random variable.  From Doob-Dynkin lemma, there exists a
deterministic function $s:\bO^T \to \Re$ such that
\[
S_T = s(Z_1,\hdots,Z_{T-1},Z_T).
\]
Set $
S(z) := s(Z_1,\hdots,Z_{T-1},z)$, for  $z\in\bO$. 
Using~\Prop{prop:linear_algebra},
\[
S_T = S(Z_T) = S_{T-1} - (U_{T-1})^\tp e(Z_T),
\]
where
\[
  S_{T-1} = \frac{1}{m+1} \sum_{z\in \bO} S(z),\quad 
 U_{T-1}(i):= - (S(i) - S_{T-1}),\quad
i=1,2,\hdots,m.
\]
Uniqueness is from the uniqueness of the decomposition. The proof is completed through induction by
repeating the procedure for $S_{T-1}\in\clZ_{T-1}$.

Now, the conditional expectation
is meaningfully defined only for sample paths $Z=z$ with 
$\sP([Z=z])>0$.  Note here that because $|\bO|=m+1$ and $T$ are
both finite, there are only finitely many---specifically $(m+1)^T$---sample paths.
Thus, $\sP([Z=z])$ is a well-defined object for each sample path, 
although it may be zero depending on the model properties.

There are two ways to address this issue: 
\begin{enumerate}
\item Assume 
  $\sP([Z=z])\geq\underline{c}^T > 0$ for all $z\in\bO^T$,
and the existence of a unique $U$ follows directly from the earlier result.
\item Adopt the convention $\frac{0}{0}=0$ to define (or extend) the
  conditional expectation for sample 
paths $Z=z$ with $\sP([Z=z])=0$. Then again, a particular
selection of $U$ follows from the above result.
\end{enumerate}
In the second case, however, there
may be other choices of $U$ such that the
representation~\eqref{eq:nonlin_predictor_rep} holds: Any two choices
will yield a representation that coincides on the set $\{z\in \bO^{T}:
\sP(Z=z)>0\}$ but may differ on the set $\{z\in \bO^{T}:
\sP(Z=z)=0\}$.
\end{proof}

\begin{example}[m=1] 
Set $
S_T^{+} = s(Z_1,\hdots,Z_{T-1},1)$ and $S_T^{-} =
s(Z_1,\hdots,Z_{T-1},0)$. 
Then $S_T^{+}, S_T^{-}\in \clZ_{T-1}$ and
\[
S_T = S_{T-1} - U_{T-1} e(Z_T),
\]
where $
S_{T-1}=0.5\left( S_T^{+} + S_{T}^-\right)$ and $U_{T-1}= -0.5\left( S_T^{+} - S_{T}^-\right)$. 
\end{example}

\section{Numerical Experiments}
\label{appx:numerics}

\subsection{Training Details for nanoGPT}
\label{appx:nanogpt-training}

We train nanoGPT~\citep{karpathy2024nanogpt}, a lightweight GPT implementation, on synthetic datasets generated from HMMs.
Training is performed using the AdamW optimizer with learning rate $10^{-3}$ and parameters $\beta_1=0.9$ and $\beta_2=0.99$ for $40,000$ iterations. 
To accelerate convergence, the first $4,000$ iterations are used for learning-rate warmup, after which the learning rate remains constant. 
The dropout rate is set to $10^{-2}$.
Representative loss curves for $d=16$ and $d=128$ are shown in Figure~\ref{fig:loss}, demonstrating that nanoGPT converges to the optimal filtering loss in both settings (the embedding dimension $\hat{d}=d$ in both cases).  
These training hyperparameters are used for all nanoGPT experiments in this work.

% We have trained nanoGPT~\citep{karpathy2024nanogpt}, a miniature version of GPT (generative Pre-trained Transformers) on synthetic datasets generated from Hidden Markov Models. 
% To keep things simple, we have only used one attention head and one layer of attention followed by a feedforward layer and subsequent layer normalization. 
% The dropout is set to a value of $10^{-2}$. We have used AdamW optimizer with a learning rate of $10^{-3}$ and $\beta_1 = 0.9, \beta_2 = 0.99$ and trained the model for $40,000$ iterations. 
% Further, to speed up convergence, we have used initial $4,000$ iterations as warmup iterations following which the learning rate is kept constant. Sample loss curves can be found in Figure~\ref{fig:loss} for $d = \hat{d} = 16$ and $d = \hat{d} = 128$ from which we can observe that nanoGPT converges in loss to the optimal filter loss for both the cases. 
% The training settings have been kept the same for all the nanoGPT experiments performed in this work.

\begin{figure}[t]
  \centering
  \includegraphics[width=\linewidth]{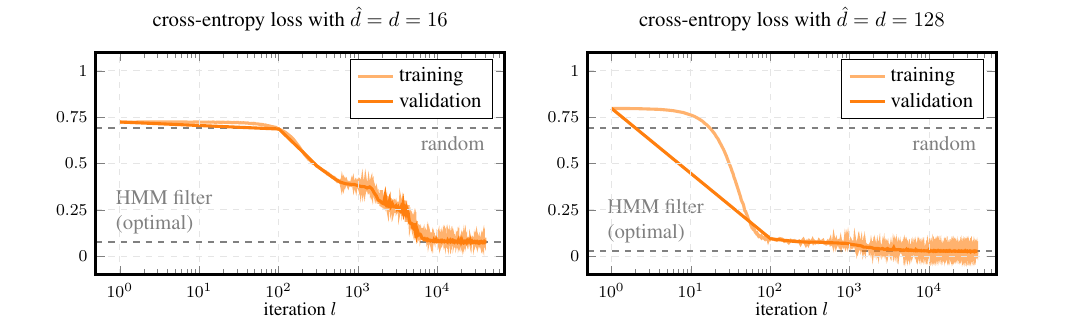}
  \caption{%
    Training and validation loss curves for the transformer model. 
    The loss converges to the optimal filter loss in both settings of $d = \hat{d} = 16$ (left) and $d = \hat{d} = 128$ (right).
    The colored solid lines represent the training (light) and validation (dark) loss curves, while the dashed lines indicate the optimal filter loss (dashed line) and the loss of uniformly random guessing (dashed line).
  }
  \label{fig:loss}
\end{figure}

\subsection{Model perturbation study}
\label{appx:robust}

In this experiment, we investigate the robustness of the dual filter control weights and nanoGPT attention weights under model perturbations. 
The model parameters are perturbed via a convex combination of the nominal parameters and a uniform distribution, where the perturbation level $\varepsilon \in [0,1]$ determines the deviation from the nominal model. 
We consider perturbations to either the transition matrix or the emission matrix (but not both together). The transition matrix perturbation are as follows
\begin{subequations}
  \label{eq:perturbed-model}
\begin{equation}
  \label{eq:perturbed-transition-model}
  A(x, x') = (1 - \varepsilon) A_{\text{nominal}}(x, x') + \frac{\varepsilon}{d}, \;\; x, x' \in \bS,
\end{equation}
and the emission matrix is perturbed as follows
\begin{equation}
  \label{eq:perturbed-emission-model}
  C(x, z) = (1 - \varepsilon) C_{\text{nominal}}(x, z) + \frac{\varepsilon}{m+1} , \; \; x \in \bS,\;\; z\in\bO.
\end{equation}
The nominal model is illustrated in \Fig{fig:model-observation}. 
In~\eqref{eq:perturbed-transition-model} and~\eqref{eq:perturbed-emission-model}, the second terms correspond to uniform distributions added to the respective model parameters.
\end{subequations}
As $\varepsilon$ increases, the model deviates further from the nominal, smoothing the deterministic cycle structure, allowing us to assess how the control and attention patterns change.  \Fig{fig:robust} depicts the control and attention patterns under three representative perturbation levels: $\varepsilon=0.01$, $\varepsilon=0.1$, and $\varepsilon=0.2$.
% Figure~\ref{fig:robust-emission} depicts the same for perturbation of the emission matrix using~\eqref{eq:perturbed-emission-model}. 
The nanoGPT is trained separately for each perturbation level to ensure that the attention patterns are learned under the corresponding data generated from the perturbed model conditions. 

The results of this study indicate that the structure of the optimal control weights are sensitive to model perturbations. 
Even small perturbations to the state transition parameter $A$ cause the dual filter's control weights to shift from a vertical, event-driven pattern to a sparse but diagonal one. 
This diagonal structure indicates that the filter is aggregating evidence across multiple past cycles rather than relying solely on a single branching moment.  
In contrast, the attention weights of nanoGPT remain sparse and vertical across these perturbations. 
Despite this structural difference, the transformer demonstrates inherent robustness by maintaining near-optimal loss that closely matches the dual filter.
On the other hand, perturbations to the emission parameter $C$ have less impact on the weights, which remain focused on the informative time steps where $Z_t=1$ across all perturbation levels, albeit with increasing diffuseness as $\varepsilon$ grows.

% We can note a few things about the obtained results. The attentions for all perturbation levels remain focused on the informative time steps where $Z_t=1$ across all perturbation levels, albeit with increasing diffuseness as $\varepsilon$ grows, indicating a graceful degradation in attention precision. In contrast, the dual filter's control weights shift from a sparse, event-driven pattern to a more diffuse, periodic pattern that reflects the underlying cycle structure of the data rather than the specific informative observations even with a tiny perturbation ($\varepsilon=0.01$). However, as the perturbation level increases, the control weights still capture the temporal periodicity of the cycles, albeit in a less precise manner, indicating that while the dual filter is sensitive to model accuracy, it still retains some ability to exploit coarse-grained temporal regularities in the data under perturbation. This experiment highlights the inherent robustness of the transformer's attention mechanism to model misspecification, as it continues to identify salient observations even under significant perturbation, while the dual filter's control strategy changes rapidly, shifting from precise event localization to a pattern that captures broader temporal structures rather than specific informative events.

\begin{figure}[t]
  \centering
  \begin{minipage}{0.49\linewidth}
    \centering
    \includegraphics[width=\linewidth]{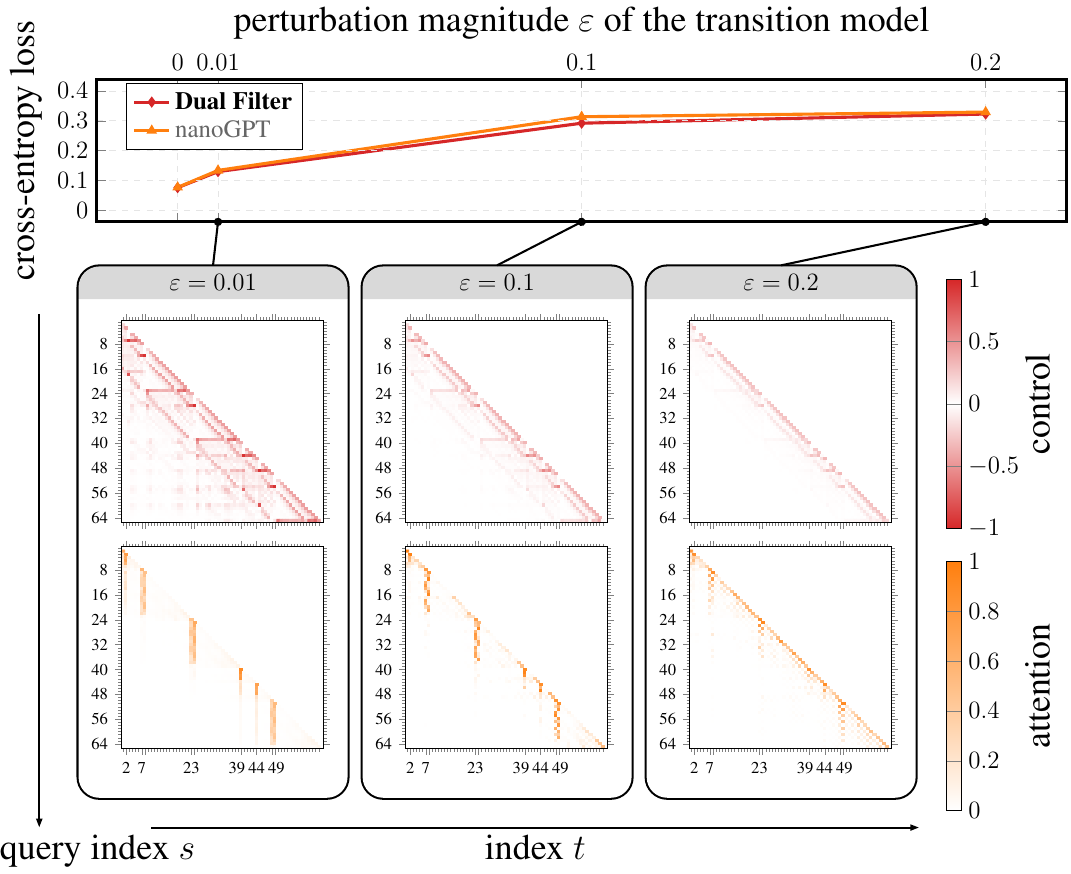}
  \end{minipage}\hfill
  \begin{minipage}{0.49\linewidth}
    \centering
    \includegraphics[width=\linewidth]{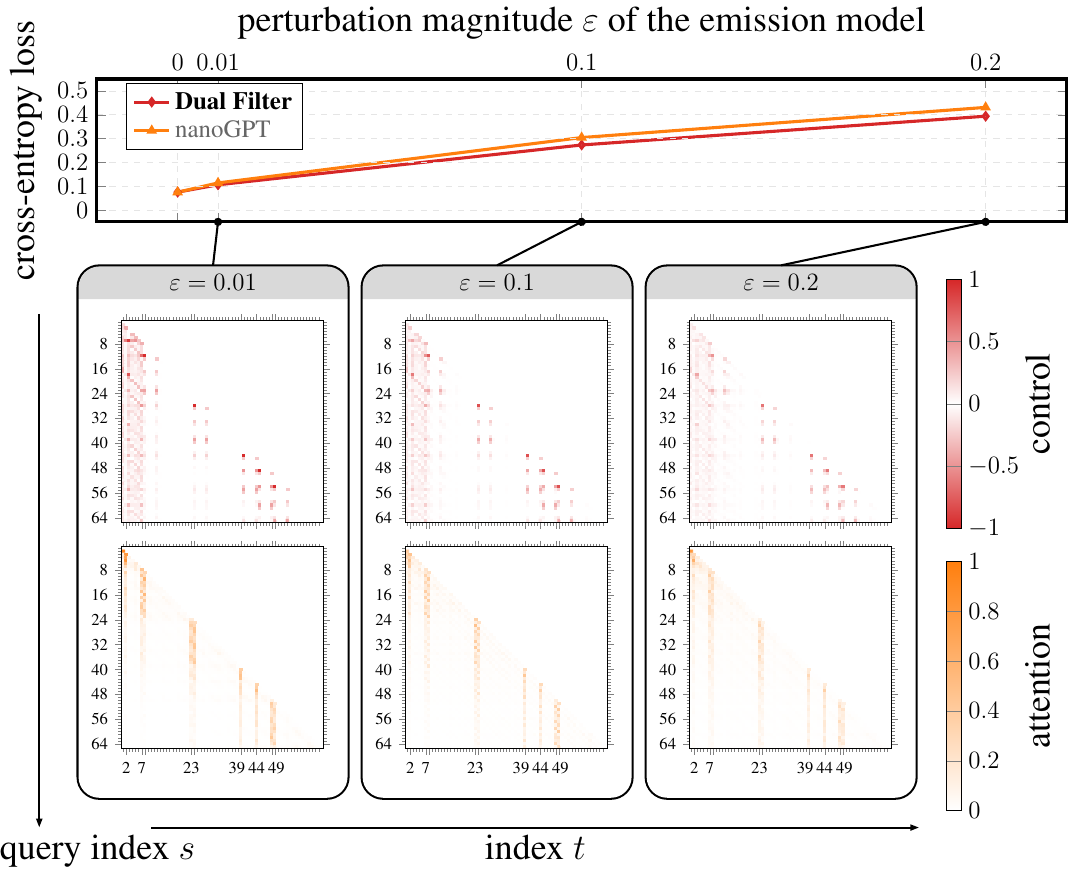}
  \end{minipage}
  \caption{Control and attention patterns shift under model perturbation. 
    The model parameters are perturbed via a convex combination of the nominal parameters and a uniform distributed probability matrix, with a perturbation level of $\varepsilon\in[0,1]$ following \eqref{eq:perturbed-model}.
    The perturbation from left to right is $\varepsilon=0.01$, $\varepsilon=0.1$, and $\varepsilon=0.2$, and the perturbation of the transition matrix is shown on the left while the perturbation of the emission matrix is shown on the right.
    The cross-entropy loss of the learned nanoGPT remains near-optimal across all perturbation levels, closely matching the dual filter's loss, indicating that the transformer maintains robust performance even as the data generated from the perturbed model.
    The top rows show the dual filter control weights computed with the perturbed model parameters, while the bottom rows show the attention matrix from a trained transformer under the same model perturbation.
    The major ticks in the index axis of the heatmaps correspond to the time steps where $Z_t=1$; the minor ticks correspond to the time steps where $Z_t=0$.
    }
  \label{fig:robust}
\end{figure}

% \begin{figure}[t]
%   \centering
%   \includegraphics[width=\linewidth]{figures/robust-emission.pdf}
%   \caption{Control and attention patterns under model perturbation. 
%     The model parameters are perturbed via a convex combination of the exact parameters and an uniform distributed transition matrix, with a perturbation level of $\varepsilon\in[0,1]$.
%     The perturbation from left to right is $\varepsilon=0.01$, $\varepsilon=0.1$, and $\varepsilon=0.2$.
%     (top) Dual filter control weights $u_t$ computed with perturbed model parameters: the control pattern ...
%     (bottom) Attention matrix $a_{T,t}$ from a trained transformer under model perturbation: the attention remains focused on time steps where $Z_t=1$, but the attention weights become more diffuse as the perturbation level increases.}
%   \label{fig:robust-emission}
% \end{figure}

\subsection{A higher-dimensional system}
\label{appdx:high_dim}

The two-cycle HMM from \Sec{sec:numerics} is used with $d = 128$, $T = 256$, and $q=4$.  A nanoGPT is trained on this HMM data, consistent with the training details in \Sec{appx:nanogpt-training}. \Fig{fig:high-dim} depicts a sample observation trajectory, and the corresponding control and attention weights.  Similar to the lower-dimensional case in the main text, the controls and attentions exhibit a sparse pattern that is non-zero at the time indices when $1$s are observed in the data. An interesting observation is that the controls in this case also show a periodic pattern that is not observed in the lower-dimensional case.

\begin{figure}[t]
  \centering
  \includegraphics[width=\linewidth]{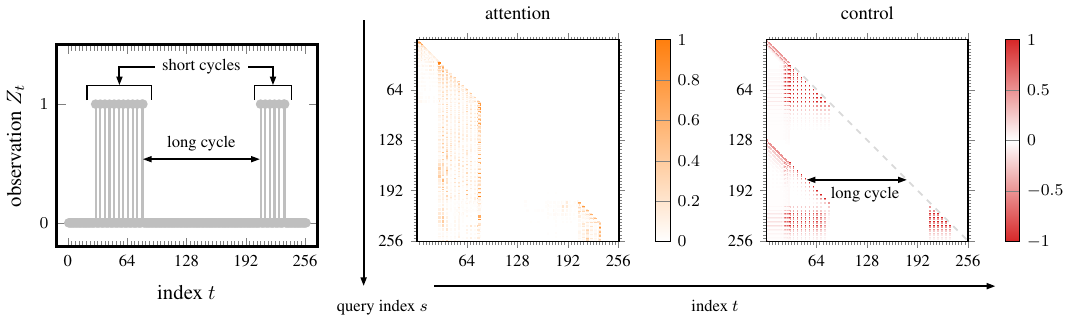}
  \caption{
    A sample observation trajectory and the corresponding control and attention patterns for the higher-dimensional system.
    The figure on the left shows a sample observation trajectory, where it shows both the shorter cycle and the longer cycle.
    The middle figure shows the attention weights from the trained nanoGPT, and the right figure shows the control weights from the dual filter.
    Both the attention and control weights exhibit a sparse pattern that is non-zero at the time indices when $1$s are observed in the data. 
    Additionally, the control weights also show a periodic pattern that reflects the underlying longer cycle structure of the data.
    The gray dashed line indicates the diagonal in the control heatmap.
  }
  \label{fig:high-dim}
\end{figure}

% \newpage
% \input{checklist-complete.tex}

\end{document}